\documentclass[10pt,letterpaper]{article}
\usepackage{booktabs,amsmath,amssymb,amsthm,array,tabularx,graphicx,float}
\usepackage{mmla_technical_report}
\usepackage{orcidlink}
\usepackage{placeins,needspace}
\graphicspath{{figures/}}
\newcommand{\mmla}{\textsc{MMLA}}
\newcommand{\sg}{\operatorname{sg}}
\newcommand{\rank}{\operatorname{rank}}
\newcommand{\E}{\mathbb{E}}
\newcommand{\R}{\mathbb{R}}
\newcommand{\cL}{\mathcal{L}}
\newcolumntype{Y}{>{\raggedright\arraybackslash}X}
\theoremstyle{definition}
\newtheorem{definition}{Definition}[section]
\theoremstyle{plain}
\newtheorem{proposition}[definition]{Proposition}

\title{LARC: Low-Rank Adaptive Residual Connections\\for Learning in Frozen Models}
\author{%
\textbf{Junyi Zou}\textsuperscript{1,*\,\orcidlink{0009-0009-1367-7428}}\quad\textbullet\quad
\textbf{Avrova Donz}\textsuperscript{1,2,*\,\orcidlink{0009-0009-0100-0719}}\\[-0.06em]
{\fontsize{7.3pt}{8.2pt}\selectfont \textsuperscript{1}MMLA-org\quad
\textsuperscript{2}Communication University of China (CUC), No. 1 Dingfuzhuang East Street, Chaoyang District, Beijing 100024, China}\\[-0.10em]
{\small \textsuperscript{*}Equal contribution}\\[-0.10em]
{\scriptsize Email:\enspace\href{mailto:lior.j.zou@gmail.com}{\nolinkurl{lior.j.zou@gmail.com}}
\enspace$\cdot$\enspace\href{mailto:avrovadonz@icloud.com}{\nolinkurl{avrovadonz@icloud.com}}}}
\date{November 9, 2026}
\hypersetup{pdftitle={LARC: Low-Rank Adaptive Residual Connections for Learning in Frozen Models},
pdfauthor={Junyi Zou and Avrova Donz},
pdfsubject={Low-rank learning state and the post-training MMLA realization in Orbit},
pdfkeywords={low-rank adaptation, residual connections, fast weights, meta-learning, reasoning-time training}}

\begin{document}
\makemmlatitle{%
Low-Rank Adaptive Residual Connections (LARC) give a frozen model a compact
numerical state that can learn from feedback. The map $h+BAh$ changes the hidden
input to a pretrained mixer. An episode copies a learned initialization $\rho$
into a private state $\Phi$, updates its factors when feedback arrives, and
resets to the bound $\rho$. We analyze the factor-space updates and trace this
input-side policy carrier into Orbit, the post-training realization of
Memory-Mediated Learning Architecture studied here. In Orbit's released
default, the learned residual shapes predictions while memory adaptation
changes the retained evidence.

\medskip\noindent The controlled studies use a rank-4 residual with 12,288
factor entries on frozen MiniCPM5-1B-SFT. Two feedback updates reduce expected
query execution error by 24.65 and 36.65 percentage points relative to the
respective trained reset states in a four-candidate development task with 16
parameter groups and three paired seeds. A direct support-loss rule reaches
0.78125\% error. In repository-balanced chronological replay of public
continuous-integration jobs, retaining updates raises half-Brier loss from
0.1274 to 0.1808. A follow-up intervention observes same-batch non-descent and
inconsistent future benefit from shrinking steps. These comparisons measure
when numerical state helps a later read and when the trained reference is
more effective.
}

\begin{center}
\begin{minipage}{.96\textwidth}
\includegraphics[width=\linewidth]{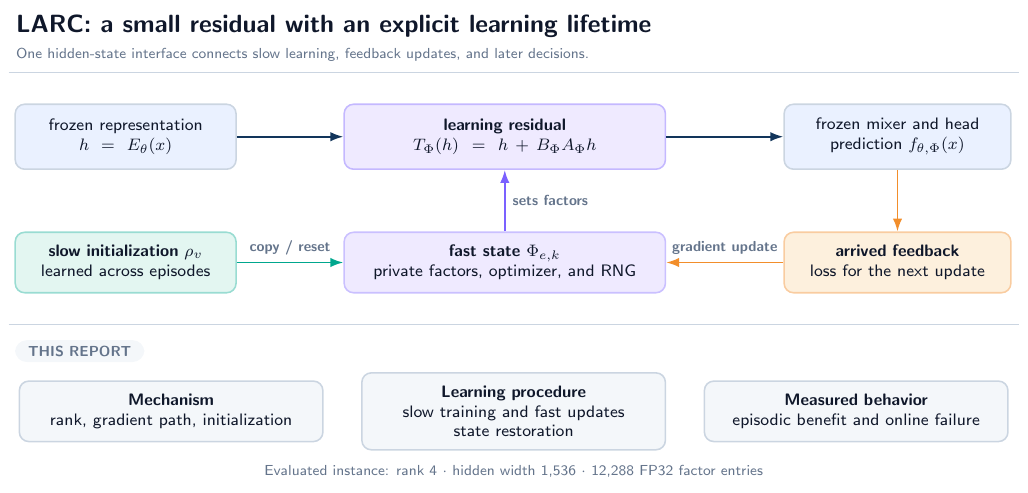}
\captionsetup{hypcap=false}
\captionof{figure}{\textbf{LARC at a glance.} Slow training determines the starting factors; arrived feedback updates a private fast state; later predictions read that state through a frozen mixer. Copy, reset, and resume give these numerical changes a defined lifetime.}
\label{fig:overview}
\end{minipage}
\end{center}
\clearpage

\section{Introduction}
Feedback often arrives while a task is still in progress. A program execution reveals that a proposed transformation is wrong. A completed workflow provides evidence about a later job. The next decision can use this observation through additional context, a stored record, or a change in the parameters that participate in the computation. LARC implements this last path through a small residual attached to a hidden representation. The question is whether retaining its feedback updates improves a later read, and under which learning conditions.

Consider a learner given four complete programs, one of which is the target. It assigns a probability to each candidate, and an executor returns all four errors on the available examples. A gradient update can reduce the probability-weighted error even when the next prompt and the most likely candidate remain unchanged. This setting makes a numerical feedback effect measurable without requiring program generation. It also raises concrete implementation questions: where is the change stored, how does it reach the next prediction, how long is it retained, and what behavior returns when it is removed? A learning state needs an initialization, an update process, a reader, and a lifetime.

Residual connections preserve a reference computation through an identity path~\cite{he2015deepresiduallearningimage}. LARC\footnote{Layer-wise Adaptive Rate Control is a separate use of the acronym LARC for an optimizer; see NVIDIA's official documentation~\cite{nvidiaLARCdocumentation}.} uses LoRA's two-factor low-rank parameterization~\cite{hu2021loralowrankadaptationlarge} as a hidden-state correction $h\mapsto h+BAh$. Its factors start from a learned slow state $\rho$ and are updated as a fast state $\Phi$. Ending an episode or applying a reset restores the bound $\rho$.

The separation between $\rho$ and $\Phi$ gives the learner two timescales, following the distinction between slow parameters and fast weights~\cite{ba2016usingfastweightsattend}. Slow training shapes a starting representation across tasks. Fast updates respond to observations in the current episode. Resetting to the trained $\rho$ removes those fast updates while preserving what slow training has learned. For this input residual, learning requires a gradient through the downstream mixer: its weights stay fixed, but its response to a changed input determines how the factors should change.

\mmla{} separates a numerical policy carrier from an authoritative memory carrier and assigns them distinct writers and reset domains~\cite{zou2026mmlamemorymediatedlearningarchitecture}. LARC supplies a concrete realization of the policy carrier. Its hidden-to-hidden interface can precede different mixers. The post-training route starts with a pretrained substrate and learns the state and interfaces around it. Orbit realizes this route with an input LARC and a memory system. A jointly trained Transformer can instead learn its sequence and evidence readers together with the adaptive operators. An independent recurrent model can carry the same state distinction through a different sequence computation.

Our central empirical comparison trains two initializations using the same feedback, queries, residual placement, and paired random seeds. One objective evaluates queries at the slow initialization. The other evaluates them after two feedback updates, using a first-order MAML-style estimator. For each trained seed, retaining real-feedback updates reduces expected query error on held-out development parameter groups. The advantage of one training objective over the other is less consistent across seeds, and the direct symbolic rule remains much more accurate than either arm. We then follow the static initialization into a task-matched continuous-integration study, where preserving online updates is harmful relative to reset on the recorded queue. The contrast shows why feedback responsiveness, the quality of the starting policy, and sustained online benefit need separate measurements.

This report specifies and implements a rank-4 input residual with slow initialization and episode-local fast state, including its gradient path through a frozen mixer and its reset and resume operations. It compares static and post-adaptation training with paired initializations, then measures real and sham feedback crossed with keep and reset reads. The factor-space analysis explains how an input residual changes the computation and why identical factor products can have different subsequent updates. Six episodic initializations and three task-matched initializations were trained and saved along this sequence.

\mmlanote{Sections~\ref{sec:method}--\ref{sec:causal} develop the factor state, its dynamics, and the feedback comparisons. Sections~\ref{sec:episodic}--\ref{sec:diagnostic} report the component studies. Section~\ref{sec:orbit} follows the same input operator through the Orbit implementation.}

\section{LARC as a Learning State}
\label{sec:method}
\subsection{Residual interface}
Let $h\in\R^d$ be a hidden vector at an insertion boundary and let $F_\theta$ denote the downstream frozen computation. A rank-$r$ LARC applies
\begin{equation}
 T_\Phi(h)=h+B_\Phi A_\Phi h,
 \qquad A_\Phi\in\R^{r\times d},\quad B_\Phi\in\R^{d\times r},
 \qquad f_{\theta,\Phi}(x)=F_\theta(T_\Phi(E_\theta(x))).
 \label{eq:residual}
\end{equation}
The first map, $A_\Phi$, reads $r$ learned linear combinations of the hidden coordinates. The second, $B_\Phi$, writes their contribution back into the $d$-dimensional representation. Adding the original $h$ preserves the identity path. The correction lies in the column space of $B_\Phi$, although changing that correction can alter many outputs of the nonlinear mixer that follows it. Low rank describes the inserted linear correction, rather than the rank of the complete input-to-output computation.

For a sequence, the same factors act tokenwise. The factors are shared within one episode and isolated across episodes. We use $r=4$, $d=1536$, and a fixed residual scale of one. There is no bias or nonlinearity inside this residual. The 12,288 factor entries require 49,152 bytes in FP32, excluding optimizers and runtime activations. Applying the two linear maps uses $2dr=12{,}288$ multiply-accumulates per token, in addition to the frozen model computation and the residual addition.

\Needspace{5\baselineskip}
\begin{definition}[LARC state]
The slow state $\rho=(A_\rho,B_\rho,v)$ contains the learned initialization and its version $v$. The fast state $\Phi_{e,k}=(A_{e,k},B_{e,k})$ is a private copy for episode $e$ after $k$ feedback updates. Each episode binds one version of $\rho$, one update rule, and private optimizer and random-number state. Its reset restores that bound initialization.
\end{definition}

This definition fixes the reference against which adaptation is measured. Before training, we sample $A_\rho$ from $\mathcal N(0,0.02^2)$ and set $B_\rho=0$. After training, both factors may be nonzero. Starting a new episode copies the trained factors, so the initial episode behavior includes slow learning. Numerical state can be retained across feedback steps while the observed prompt remains fixed.

\begin{figure}[htbp]
\centering\includegraphics[width=\linewidth]{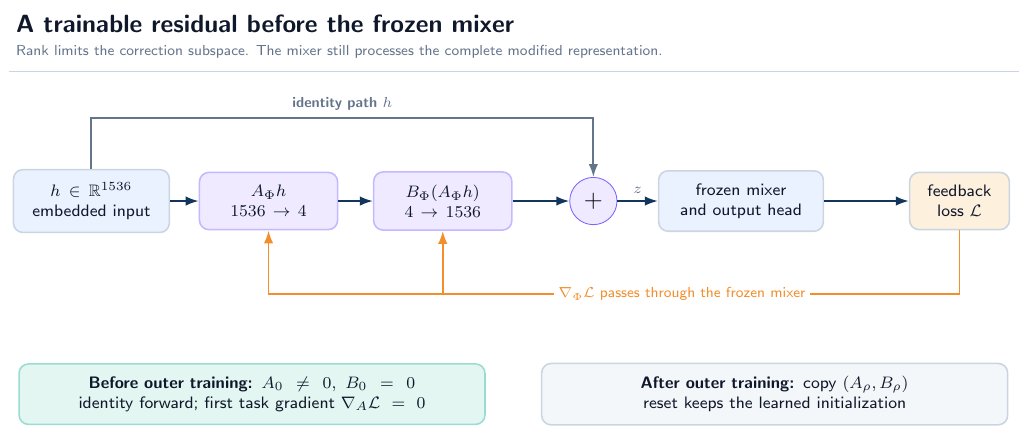}
\caption{\textbf{The input residual and its learning path.} The two factors read a four-dimensional projection and write a correction back into the full hidden representation. The orange path denotes backpropagation to the factors through the frozen downstream computation. Zero $B$ makes the initial forward map the identity; an episode reset later restores the trained factors.}
\label{fig:interface}
\end{figure}

\subsection{What changes, and on which timescale?}
Three parameter collections participate in the evaluated learner. The common substrate contains the pretrained weights and the always-enabled semantic adapter $\omega$. Their values are fixed throughout the study. The slow LARC factors $\rho$ change across outer training batches. The fast factors $\Phi$ change within an episode or online session. A factor tensor belongs to exactly one of these roles at a time: a fast update operates on an owned copy, rather than modifying the slow reference in place.

\begin{table}[htbp]
\centering\small
\begin{tabularx}{\linewidth}{lYYY}\toprule
State & What it contains & When it changes & What a fast reset does\\\midrule
Frozen substrate & Backbone and semantic adapter $\omega$ & Fixed during these experiments & Leaves it unchanged\\
Slow $\rho_v$ & Learned factors $(A_\rho,B_\rho)$ and version & One outer update after a batch closes & Uses it as the reference\\
Fast $\Phi_{e,k}$ & Episode factors and private update state & Arrived-feedback update & Copies the bound $\rho_v$\\
Explicit memory $M$ & Completed cases in the CI study & Real-label arrival under its own rule & Keeps its separate lifetime\\\bottomrule
\end{tabularx}
\caption{\textbf{The states read by the evaluated system.} Explicit memory is disabled in the episodic study. In CI it provides the same causal case history to the compared policy states; resetting the numerical residual does not erase that history.}
\label{tab:states}
\end{table}

This distinction is useful even without meta-learning. An ordinarily trained residual can be copied into an episode and adapted from feedback. Meta-learning changes how the reference is obtained; it does not create the distinction between a reference and a temporary state. Conversely, a successful outer training loss does not show that keeping subsequent fast updates helps. That question requires the keep/reset comparison developed in Section~\ref{sec:causal}.

\subsection{Two timescales and two outer objectives}
We use an initialization-based meta-learning formulation~\cite{finn2017modelagnosticmetalearningfastadaptation} to compare two outer objectives. For support feedback $S_e$ and an outer query $Q_e$, the fast update is
\begin{equation}
 \Phi_{e,0}=\operatorname{copy}(\rho_v),\qquad
 \Phi_{e,k+1}=\Phi_{e,k}-\eta\nabla_{\Phi_{e,k}}\cL_{S_e}(\Phi_{e,k}).
 \label{eq:inner}
\end{equation}
Our episodic experiments use two steps of SGD, $\eta=0.1$, with no momentum, decay, or inner gradient clipping. All episodes in an outer batch start from the same version of $\rho$. Their factors have distinct storage and random-number streams. Only after every episode finishes do we aggregate query losses and update the slow factors.

We compare
\begin{align}
 J_{\mathrm{static}}(\rho)&=\E_e\big[\cL_{Q_e}(\rho)\big],\\
 J_{\mathrm{adapted}}(\rho)&=\E_e\big[\cL_{Q_e}(U^2(\rho,S_e))\big].
 \label{eq:objectives}
\end{align}
Static training computes the same support updates for matched work and observations, but its query loss is evaluated at $\rho$. Adapted training evaluates the query at $\Phi_{e,2}$. We implement its first-order gradient through
\begin{equation}
 \widetilde\Phi_{e,2}=\rho+\sg(\Phi_{e,2}-\rho),\qquad
 \widehat{\nabla_\rho J}_{\mathrm{adapted}}
 =\frac{1}{|\mathcal B|}\sum_{e\in\mathcal B}
 \nabla_\Phi\cL_{Q_e}(\Phi)\big|_{\Phi=\Phi_{e,2}}.
 \label{eq:fo}
\end{equation}
Equation~\eqref{eq:fo} implements a first-order MAML-style estimator~\cite{nichol2018firstordermetalearningalgorithms}. Here $\sg$ is stop-gradient: the forward value is the adapted state, and the backward Jacobian to $\rho$ is the identity. The support trajectory determines where the query gradient is measured; that gradient is passed to the corresponding slow factor entries without differentiating through the two support updates. An exact MAML gradient would include their Jacobians. Detaching the fast state without the identity connection would instead disconnect the outer gradient from $\rho$.

Both objectives use AdamW~\cite{loshchilov2019decoupled} with learning rate $10^{-3}$, $(\beta_1,\beta_2)=(0.9,0.999)$, $\epsilon=10^{-8}$, zero weight decay, and outer gradient-norm clipping at one. Each outer batch contains two episodes. We train exactly 256 outer updates for each of three paired initializations in each objective, retaining all six final states.

\begin{figure}[htbp]
\centering\includegraphics[width=\linewidth]{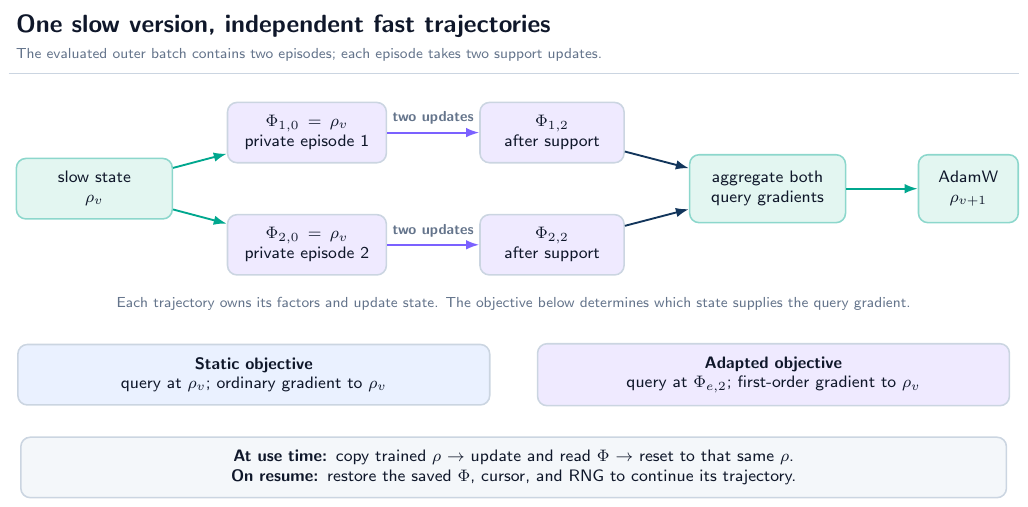}
\caption{\textbf{Slow learning and fast adaptation.} Every episode in a batch begins at the same slow version and follows its own support trajectory. The objective chooses whether its query is evaluated at $\rho_v$ or at the adapted factors. Query gradients are then aggregated into one slow update. At use time, the trained slow state stays fixed; fast resets and checkpoint restoration perform different operations.}
\label{fig:lifetimes}
\end{figure}

\subsection{Gradient flow, batching, and restoration}
The model weights are frozen while the downstream function remains differentiable with respect to its input. With $z=T_\Phi(h)$, an input-side gradient contains the Jacobian $\partial F_\theta(z)/\partial z$. Wrapping this computation in a no-gradient region removes the learning path. Our implementation caches only the frozen embeddings before LARC and uses non-reentrant activation checkpointing for the downstream mixer. The output head is evaluated at answer positions. Padding masks and position indices remain specific to each sequence.

An episodic physical batch has eight rows: two independent episodes, each with four candidate continuations. Each row selects its episode's own factors. A permutation of episode rows must permute the corresponding predictions without changing them. This isolation is also required for optimizers and random-number state; ordinary batching must not turn independent learners into one shared learner.

A slow checkpoint stores actual factor values, version, original initialization, outer optimizer, random-number state, model identity, and configuration. A live-session checkpoint additionally stores the fast factors, feedback cursor, pending feedback, and any separate explicit memory. Reset and resume serve different purposes. Reset restores the bound slow initialization; resume restores the current fast trajectory. We tested midpoint restoration and final reload as part of the recorded training and replay runs.

For example, an episode saved after its first support update must resume at $\Phi_{e,1}$ before taking its second update. Replacing that saved state with $\rho_v$ would replay a different trajectory. A reset intervention deliberately makes this replacement just before the designated query. In a delayed-feedback session, restoration must also retain which jobs have already been predicted and which labels have arrived, because those determine the information available to the next update.

\section{Properties of the Residual and Its Updates}
\label{sec:algebra}
The following statements describe the real-valued residual in Equation~\eqref{eq:residual}. They explain the parameterization and initialization, while the implemented model combines FP32 factors with BF16 activations.

\subsection{Relation to LoRA: placement and learning lifetime}
LoRA parameterizes a weight correction~\cite{hu2021loralowrankadaptationlarge}, writing a linear layer as $(W+\Delta W)h$ with $\Delta W=UV$. LARC uses the same low-rank factorization to transform its input. If the next operation is a linear map $W\in\R^{m\times d}$, then
\begin{equation}
 W T_\Phi(h)=W(h+B_\Phi A_\Phi h)=Wh+(WB_\Phi)A_\Phi h.
 \label{eq:lora}
\end{equation}
The effective correction $\Delta W=(WB_\Phi)A_\Phi$ has rank at most $r$, and its columns lie in the image of $W$. Conversely, for any rank-at-most-$r$ correction with columns in that image, choose a rank factorization $\Delta W=UV$ whose left factor lies in the image, padding with zeros to width $r$ if needed. Solving $WB_\Phi=U$ and setting $A_\Phi=V$ realizes that correction. When $W$ has full row rank, every rank-at-most-$r$ weight correction is representable. A square invertible $W$ is a special case, with $B_\Phi=W^{-1}U$. When $W$ is not full row rank, corrections outside its image remain inaccessible. These statements concern forward representations; the factor coordinates and their gradient updates can still differ.

Our evaluated residual precedes the full nonlinear mixer. The mixer must be evaluated on the changed representation; the linear identity above does not generally merge this transformation into an arbitrary internal LoRA site. Table~\ref{tab:lora} separates the placement of the correction from its learning lifetime.

\begin{table}[htbp]
\centering\small
\begin{tabularx}{\linewidth}{lYY}\toprule
Choice & Evaluated input LARC & Weight-space LoRA\\\midrule
Forward correction & $F_\theta((I+BA)h)$ & $(W+UV)h$ at selected linear maps\\
Linear overlap & $\Delta W=(WB)A$ & Free factored $\Delta W=UV$\\
Placement here & One hidden-to-hidden map before the full mixer & Depends on the chosen weight sites\\
Learning lifetime & Trained $\rho$; private, adapted $\Phi$; reset to $\rho$ & Can use the same initialization, adaptation, and reset procedure\\\bottomrule
\end{tabularx}
\caption{Placement and lifetime are separate choices. The linear overlap is given by Equation~\eqref{eq:lora}; the evaluated LARC precedes the full nonlinear mixer.}
\label{tab:lora}
\end{table}

\subsection{Identity initialization and the first learning step}
Let $C=BA$ and write the differentiable task loss as $\ell(C)$, with matrix gradient $G_C=\nabla_C\ell(C)$ under the Frobenius inner product. The chain rule gives
\begin{equation}
 \nabla_A\ell=B^\top G_C,\qquad \nabla_B\ell=G_C A^\top.
 \label{eq:gradient}
\end{equation}
\begin{proposition}[First update from an identity residual]
With simultaneous SGD, initial $B_0=0$, and initial $A_0$, the first update satisfies
\begin{equation}
 A_1=A_0,\qquad B_1=-\eta G_{C,0}A_0^\top,\qquad
 C_1=-\eta G_{C,0}A_0^\top A_0.
 \label{eq:first}
\end{equation}
If $A_0$ has independent zero-mean entries of variance $\sigma^2$, and $G_{C,0}$ depends on the factors only through $C_0=0$, then
$\E[C_1]=-\eta r\sigma^2G_{C,0}$.
\end{proposition}
\begin{proof}
Substitute $B_0=0$ into Equation~\eqref{eq:gradient}. Then multiply the updated factors and use $\E[A_0^\top A_0]=r\sigma^2I$.
\end{proof}
The first step learns within the row space selected by $A_0$. In expectation, that initial update is proportional to the full matrix gradient, although each realization remains rank constrained. The zero first-step gradient of $A$ is therefore expected. A nonzero $B$ gradient additionally requires $G_{C,0}A_0^\top\ne0$; random initialization does not replace checking the actual task gradient. Setting both factors to zero yields zero task gradients for both and leaves this bilinear branch stationary under plain SGD.

This calculation concerns the first task-gradient update. The outer optimizer, clipping, and mixed-precision arithmetic can change the numerical step. Once slow training has made $B_\rho$ nonzero, both factors can receive a task gradient at the start of a new episode.

\subsection{Rank controls capacity, not displacement}
\begin{proposition}[Rank and norm]
At any state, $\rank(BA)\le r$ and
\begin{equation}
 \|T_\Phi(h)-h\|_2\le\|B\|_2\|A\|_2\|h\|_2.
 \label{eq:norm}
\end{equation}
For two states with rank-at-most-$r$ residuals, the rank of their difference is at most $2r$. This upper bound is attainable when $d\ge2r$.
\end{proposition}
\begin{proof}
The first two assertions follow from the rank inequality for products and submultiplicativity of the operator norm. For the difference, use $\rank(C_1-C_0)\le\rank(C_1)+\rank(C_0)$. Two rank-$r$ diagonal matrices supported on disjoint sets of coordinates attain $2r$.
\end{proof}
Rank four therefore bounds each instantaneous residual to rank four, while the change between two learned residuals can have rank eight. Multiplying one factor by an arbitrarily large constant preserves its rank and increases its norm. A small state is compatible with a large representation displacement. If the downstream map is $K$-Lipschitz on the relevant segment, Equation~\eqref{eq:norm} yields an output bound multiplied by $K$. Neither a uniform $K$ nor bounded factor norms follows from the rank choice alone.

\subsection{Equal forward maps can have different learning dynamics}
For any invertible $R\in\R^{r\times r}$, the factors $(RA,BR^{-1})$ represent the same $C$. Euclidean gradient steps generally depend on this choice of coordinates. With a positive scalar $c$, consider $A'=cA$, $B'=B/c$. At the common forward map, one simultaneous SGD step changes $C$ by
\begin{equation}
 C'_+-C=-\eta\left(c^{-2}BB^\top G_C+c^2G_C A^\top A\right)
             +\eta^2G_C A^\top B^\top G_C.
 \label{eq:gauge}
\end{equation}
These factor-space properties also apply to other bilinear low-rank updates, including LoRA. The first-order terms depend on $c$ even though the current prediction does not. Initialization scale, factor norms, optimizer state, and step size are therefore part of the learning specification. Saving only the product $BA$ is enough to reconstruct a forward map, but is generally insufficient to reproduce the next factor-space update.

\begin{proposition}[Equal products, different SGD updates]
\label{prop:product-state}
For every step size $\eta>0$, there are scalar factor states with identical
forward maps and a common differentiable loss whose next products differ
under simultaneous SGD. A deterministic update that takes the current
product, loss, and step size as its inputs cannot reproduce all such factor
trajectories.
\end{proposition}
\begin{proof}
Use $(A,B)=(1,0)$ and $(A,B)=(2,0)$, and take $\ell(C)=C$.
Both products are zero and $G_C=1$ at either state. Equation~\eqref{eq:first}
gives next products $-\eta$ and $-4\eta$. The inputs to a product-only rule
coincide while the required outputs differ.
\end{proof}

For an $L$-smooth factor-space objective, the standard descent bound is
\begin{equation}
 \ell(\Phi-\eta\nabla\ell(\Phi))
 \le\ell(\Phi)-\eta\left(1-\frac{L\eta}{2}\right)\|\nabla\ell(\Phi)\|^2.
 \label{eq:descent}
\end{equation}
It requires smoothness along the step and $0<\eta<2/L$ to obtain a strict decrease for a nonzero gradient. Rank supplies neither condition. Moreover, a decrease on the arrived support batch does not determine the loss on later jobs. The experiments measure both quantities.

\section{Measuring Feedback Adaptation}
\label{sec:causal}
We evaluate numerical adaptation using a feedback intervention crossed with a state-retention intervention. Real feedback preserves its binding to candidate programs. Sham feedback independently permutes the four observed losses uniformly at each inner step, including natural fixed points. Both branches start from the same trained initialization, take the same number of updates, and then either keep the fast factors or reset to their own $\rho$ immediately before the query read.

Let $R_{z,a}$ be the query risk for feedback condition $z\in\{\mathrm{real},\mathrm{sham}\}$ and read state $a\in\{\mathrm{keep},\mathrm{reset}\}$. Define
\begin{align}
 G&=R_{\mathrm{real},\mathrm{reset}}-R_{\mathrm{real},\mathrm{keep}},\\
 D&=G-\left(R_{\mathrm{sham},\mathrm{reset}}-R_{\mathrm{sham},\mathrm{keep}}\right).
 \label{eq:gd}
\end{align}
Positive $G$ measures benefit from retaining real-feedback state relative to reset. Positive $D$ measures the corresponding difference from the permutation intervention. The two reset reads are identical for a deterministic query with the same substrate and initialization, so they can share one physical evaluation. Their logical identities remain distinct.

The query read receives the public prompt, candidate programs, frozen model, and designated factor state. It receives no feedback transcript, inner optimizer, live inner random stream, cached post-residual hidden representation, or stale key--value cache. Explicit memory is disabled in the episodic study. These conditions make retention of $\Phi$ the manipulated path from feedback to the query. A text-feedback baseline is evaluated separately with the feedback rendered in its prompt.

\begin{figure}[htbp]
\centering\includegraphics[width=\linewidth]{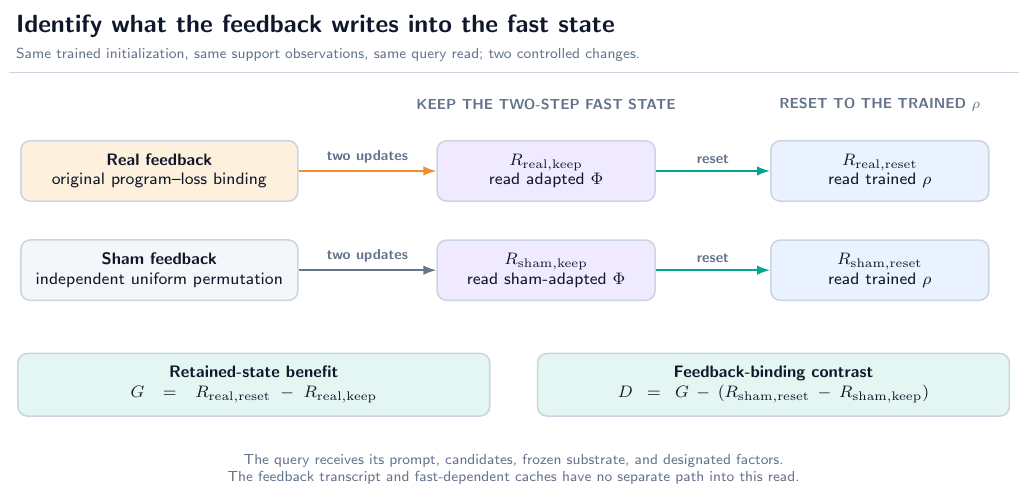}
\caption{\textbf{Feedback binding crossed with state retention.} Both feedback branches execute two updates. The reset read then restores its own trained initialization immediately before the query. With the remaining read inputs fixed, $G$ measures the effect of retaining real-feedback state, while $D$ compares this effect with independently permuted feedback.}
\label{fig:interventions}
\end{figure}

Under the declared query-read controls, these interventions estimate retention and feedback-binding effects. Keeping the state tests the effect of carrying the numerical update into the query. Permuting feedback tests the effect of its binding to the candidate programs. An update can respond to the correct feedback and still move farther from a good initial policy. The reset risk therefore remains the reference in $G$, with $D$ reported alongside it.

The outer-objective contrast is
\begin{equation}
 \Delta_{\mathrm{obj}}=R^{\mathrm{static}}_{\mathrm{real},\mathrm{keep}}
                     -R^{\mathrm{adapted}}_{\mathrm{real},\mathrm{keep}}.
\end{equation}
It asks a different question from $G$ and $D$. Similar outer objectives can both yield useful feedback adaptation, and a positive difference between real and permuted updates can coexist with harmful real updates. We retain all three comparisons in the results.

\FloatBarrier
\section{Grouped Program-Selection Study}
\label{sec:episodic}
\subsection{Model and task construction}
All six models share MiniCPM5-1B-SFT revision \texttt{a60b37f1fc409c54e1e337b0723aaac6f92dfec0} and an always-enabled, previously trained semantic adapter $\omega$. This frozen adapter uses rank-16 query/value additions in layers 16--23, with scale $32/16$. It is part of the common substrate. The input LARC has rank four and is the only component trained in this study. Thus the rank-16 semantic adapter supplies fixed task semantics, while the rank-4 input residual is the learning state under evaluation.

Each episode asks the model to choose among four complete public programs in either an arithmetic or a list-processing family. Programs are instantiated from the public configuration using four fixed recipes per family. The learner receives four input/output examples and an executor returns each candidate's support error fraction. A query contains 20 new inputs to the same transformation, disjoint from support. Query outcomes supervise outer training or subsequent scoring; they are absent from the inner feedback and policy prompt.

We group arithmetic tasks by the active pair $(\mathrm{scale},\mathrm{offset})$ and list tasks by $(\mathrm{minimum},\mathrm{limit})$. All four recipes for a parameter group stay in the same split. A conservative inventory excludes historical active pairs and the reachable support of the earlier generator. The arithmetic candidate space contains 350 pairs, of which 245 remain after exclusion. The list space contains 78 pairs, of which 34 remain. We allocate 34 per family: 16 training, eight development, eight reserved final, and two preflight groups. Development consequently has 16 independent parameter groups and 64 recipe episodes. The eight recipes are familiar; the held-out variation is in their active parameters.

The policy scores each complete candidate and its end-of-sequence token using mean token log probability, then applies a temperature-one softmax over the four scores:
\begin{equation}
 s_j(\Phi)=\frac{1}{|a_j|}\sum_{t=1}^{|a_j|}\log p_{\theta,\Phi}(a_{j,t}\mid x,a_{j,<t}),
 \qquad \pi_j(\Phi)=\frac{e^{s_j(\Phi)}}{\sum_{m=1}^4e^{s_m(\Phi)}}.
\end{equation}
For externally supplied execution losses $\ell_j(S)$,
\begin{equation}
 \cL_S(\Phi)=\sum_{j=1}^4\pi_j(\Phi)\ell_j(S).
 \label{eq:risk}
\end{equation}
The query loss uses the same expectation with query execution losses. Exact enumeration gives a differentiable finite-policy objective without sampling programs or passing gradients through a discrete executor. This is a program-selection evaluation.

\subsection{A concrete feedback update}
For illustration, take arithmetic scale 2, offset 3, and clipping to $[-16,16]$. The four candidate programs are
\begin{align*}
 f_0(x)&=2x+3,& f_1(x)&=2(x+3),\\
 f_2(x)&=2\operatorname{clip}(x)+3,& f_3(x)&=\operatorname{clip}(2x)+3.
\end{align*}
Suppose the target is $f_0$ and the four support inputs are $2,24,-20,5$. The target outputs are $7,51,-37,13$. Candidate $f_1$ disagrees on all four inputs; $f_2$ and $f_3$ each disagree on the two large-magnitude inputs. The executor therefore supplies the loss vector $(0,1,\tfrac12,\tfrac12)$, together with its candidate binding. The policy uses that vector through Equation~\eqref{eq:risk}.

The derivative with respect to a candidate energy is particularly simple:
\begin{equation}
 \frac{\partial\cL_S}{\partial s_j}
 =\pi_j\bigl(\ell_j(S)-\cL_S\bigr).
 \label{eq:energy-gradient}
\end{equation}
In energy coordinates, candidates with below-average execution error are favored. The actual optimizer changes $A$ and $B$, so its step is also shaped by the model Jacobian from those factors to all four energies. A finite factor update need not realize an independent increase or decrease of each energy. This connection between an execution error and the factor-space gradient is the learning mechanism tested here.

On the query read, the updated factors change candidate probabilities while the support losses remain outside the prompt. The evaluator then applies the candidate programs to held-out inputs; for $x=3$, the target output is 9. Resetting the factors first instead reads the trained starting policy. These are the real/keep and real/reset cells. The symbolic baseline uses the same support vector directly to select $f_0$, which also explains why that baseline is strong on this finite task.

\subsection{Analysis and selection rule}
The primary endpoint is real/keep expected query error. We average recipes within each parameter group, weight the two task families equally, and pair the three training seeds. The seeds are 2026092811, 2026092812, and 2026092813, shortened to 2811--2813 in tables. Each model completes all 256 outer updates before development predictions are generated. Development query scores are opened after predictions have been saved. Reserved final cases are not constructed or read in this study.

We use 10,000 paired bootstrap resamples of complete parameter groups, stratified by family, keeping all recipes, seeds, objectives, and intervention cells together. Five 99\% percentile intervals cover $\Delta_{\mathrm{obj}}$ and the two $G,D$ pairs, with nominal Bonferroni family coverage of 95\%. The materiality rule requires a point estimate of at least three percentage points and a lower interval bound above zero. Feedback eligibility additionally requires positive $G$ and $D$ for each training seed. These intervals quantify task-group variation conditional on the three trained seeds.

Among eligible objectives, the preregistered selection rule maximizes the worst-seed value of $\min(G,D)$. Measured training-plus-adaptation cost breaks a tie within $10^{-6}$, followed by the static objective if cost also ties. This rule chooses an objective's complete seed group rather than its best individual model.

\subsection{Results}
\begin{table}[ht]
\centering\small
\begin{tabular}{lrrrrr}\toprule
Objective & 2811 & 2812 & 2813 & Mean & Seed SD\\\midrule
Static & 21.90 & 31.47 & 34.53 & 29.30 & 6.59 \\
Adapted & 20.46 & 41.93 & 11.44 & 24.61 & 15.66 \\
\bottomrule\end{tabular}
\caption{Post-adaptation development error (\%), lower is better. Each column evaluates the same 16 parameter groups. Seed SD is the sample standard deviation across the three trained initializations.}
\label{tab:episodic}
\end{table}
The adapted objective achieves 24.61\% mean error versus 29.30\% for the static objective (Table~\ref{tab:episodic}). Its paired improvement is 4.69 percentage points with a 99\% task-group interval of [1.19, 8.88], satisfying the registered development rule. Across seeds, the improvements are 1.45, $-10.46$, and 23.09 points. Thus the mean objective advantage is accompanied by a reversal for one initialization and a larger seed standard deviation.

\begin{table}[ht]
\centering\small
\begin{tabular}{lrrrrr}\toprule
Contrast & Mean & 99\% group interval & 2811 & 2812 & 2813\\\midrule
$\Delta_{\mathrm{obj}}$ & 4.69 & [1.19, 8.88] & 1.45 & -10.46 & 23.09 \\
Static $G$ & 24.65 & [19.39, 29.26] & 31.43 & 20.16 & 22.35 \\
Static $D$ & 28.36 & [21.85, 34.93] & 37.76 & 26.78 & 20.55 \\
Adapted $G$ & 36.65 & [31.65, 41.20] & 40.52 & 18.64 & 50.78 \\
Adapted $D$ & 35.77 & [30.33, 41.01] & 40.21 & 19.97 & 47.13 \\
\bottomrule\end{tabular}
\caption{Paired gains in percentage points, positive is favorable. The bootstrap unit is a complete parameter group. Training seeds are kept paired within each resample.}
\label{tab:contrasts}
\end{table}
Both objectives yield positive feedback benefits (Table~\ref{tab:contrasts}). Static training obtains $G=24.65$ and $D=28.36$ points; adapted training obtains $G=36.65$ and $D=35.77$ points. Every seed has positive $G,D$, and all four contrasts satisfy the materiality and interval rule. The selected static group has a worst-seed $\min(G,D)$ of 20.16 points, compared with 18.64 for the adapted group. Selection therefore follows stability of the feedback gain, even though adapted training has lower mean error.

The larger $G$ of the adapted objective includes a different starting point. Its reset error is 61.26\%, versus 53.95\% for static training. The 12.00-point difference between their feedback gains decomposes as
\begin{equation}
 G_{\mathrm{adapted}}-G_{\mathrm{static}}
 =\underbrace{(61.26-53.95)}_{\text{reset difference }7.31}
 +\underbrace{(29.30-24.61)}_{\text{endpoint advantage }4.69}
 \quad\text{percentage points},
\end{equation}
with the displayed values rounded. The adapted objective starts worse and ends better on average. The objective contrast measures the latter advantage; it cannot be replaced by the full difference in feedback gains.

\begin{table}[ht]
\centering\small
\begin{tabular}{lrrrrrr}\toprule
Objective & Real/keep & Real/reset & Sham/keep & Sham/reset & Text & Rule\\\midrule
Static & 29.30 & 53.95 & 57.67 & 53.95 & 61.78 & 0.78 \\
Adapted & 24.61 & 61.26 & 60.38 & 61.26 & 61.03 & 0.78 \\
\bottomrule
\end{tabular}
\caption{Descriptive development errors (\%). Text renders the program-to-feedback mapping in the prompt at fixed $\rho$. Rule chooses the lowest-support-error program, with first-index tie breaking. All methods receive the same external support information.}
\label{tab:controls}
\end{table}
The symbolic support rule reaches 0.78125\% error, considerably below either learned policy (Table~\ref{tab:controls}). Exact enumeration makes that rule a strong comparator: it can directly use the same four execution losses that drive the gradient. LARC produces a numerical feedback effect in this setting, while program selection itself has little remaining error for a learned system to remove relative to this baseline. Figure~\ref{fig:episodic-results} places the objective comparison next to the retained-state and feedback-binding gains.

\begin{figure}[H]
\centering\includegraphics[width=\linewidth]{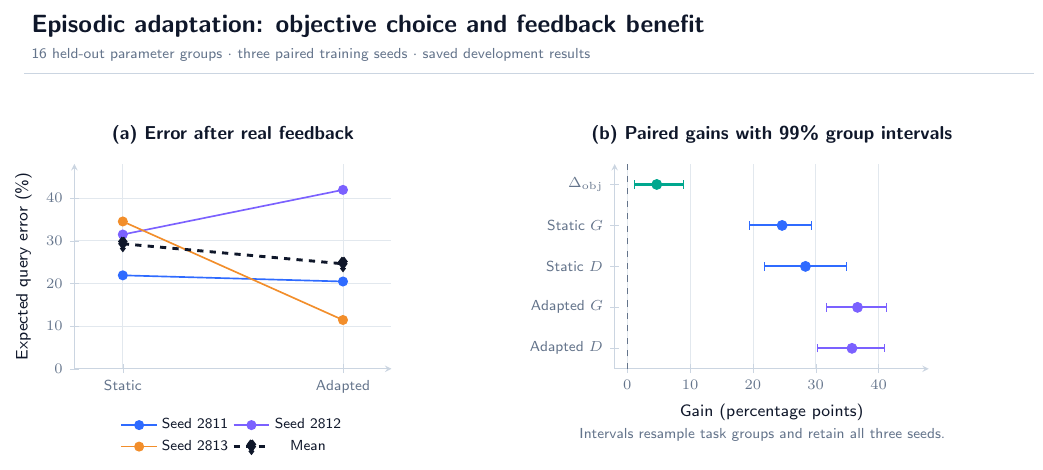}
\caption{\textbf{Two questions in the episodic study.} Left: each line joins one paired initialization under static and adapted training; the dashed black line joins their means. Right: the five recorded gains and 99\% task-group intervals. The adapted objective improves the mean query error, while both objectives support useful real-feedback adaptation. The intervals retain the three trained seeds in every resample.}
\label{fig:episodic-results}
\end{figure}

\begin{table}[H]
\centering\small
\begin{tabularx}{\linewidth}{llrrY}\toprule
Objective & Read state & Expected error & Greedy error & All query cases correct\\\midrule
Static & Reset & 53.95 & 47.76 & 42.71 \\
Static & Adapted state & 29.30 & 15.49 & 81.77 \\
Adapted & Reset & 61.26 & 61.72 & 30.21 \\
Adapted & Adapted state & 24.61 & 16.04 & 80.73 \\
\bottomrule\end{tabularx}
\caption{Additional descriptive development measures (\%), reconstructed from saved predictions. Greedy error evaluates the highest-probability supplied program; the last column is the fraction of episodes where that program passes all 20 query inputs. Group/family weights and paired seeds match the primary analysis. These are secondary measures, not free-form generation.}
\label{tab:episodic-secondary}
\end{table}
The expected-risk endpoint also differs from a greedy decision. After feedback, the static and adapted models have greedy query errors of 15.49\% and 16.04\%, respectively (Table~\ref{tab:episodic-secondary}). Their ordering is different from expected error. This distinction matters because changing probability mass across the four complete programs can improve expected risk without changing the most probable program.

\FloatBarrier

\section{Delayed Feedback in Public CI Workflows}
\label{sec:ci}
\subsection{Chronological replay and task-matched training}
The second study predicts whether a public continuous-integration job will finish as success, failure, cancelled, or other. A job is predicted at its recorded start time. Its label becomes available at completion. At a timestamp with both starts and completions, predictions precede feedback processing. Updates use the prompt saved at prediction time. This ordering prevents later outcomes or memory contents from entering an earlier prediction.

The training queue has 2,678 jobs from 60 repository--commit groups. The audit queue has 3,321 jobs from 72 groups: 83 NumPy jobs from one commit and 3,238 pandas jobs from 71 commits. The NumPy training and audit dates are May 14 and May 16, 2026; the pandas dates are September 17 and September 19. Time is ordered within each repository. Because the shared training pool includes September pandas jobs, this is a repository-local retrospective evaluation, rather than a global-calendar prospective deployment.

Each selected static initialization receives one epoch of task-matched supervised training, using both no-memory and causal-memory views of each job. The loss is categorical cross-entropy on the four candidate-label energies, weighted equally by repository, commit within repository, job within commit, and the two views. AdamW uses the episodic outer hyperparameters. Each seed completes 670 additional updates, moving from slow version 256 to 926. The resulting three slow states are fixed during audit.

The online policy uses SGD with step size 0.1 on arrived-label cross-entropy, with independent fast states for each path and repository. A separate memory stores at most 16 completed cases within 16 KiB and exposes at most two, preferring the same workflow and job name. It is updated from real completed labels in every relevant branch. The memory uses FIFO eviction; reading prefers the latest matching cases and fills any remaining slot with recent cases. Commit identifiers are exposed in the prompt, but do not impose a hard version filter. The policy-state intervention keeps online factors ($P_1$) or reads at the trained initialization ($P_0$). The memory intervention provides or omits these case records ($M_1$ or $M_0$). A permuted-label path draws a fresh, independent uniform permutation of the four labels for each arrived job using its private random stream, then applies the inverse permutation to that job's target. It retains the same real memory. There is no single fixed permutation shared by a completion batch.

All jobs completing at the same timestamp form one feedback batch. Stable identifiers determine serialization within a timestamp, without implying finer temporal resolution. For each active path, gradients are the mean cross-entropy over its arrived jobs; physical chunks accumulate that same mean. They are not divided by the number of paths or seeds and do not use the commit weights of the outer training objective. Each path then takes one functional SGD step with no momentum, decay, or gradient clipping. All factor updates finish before real cases are committed to memory.

At audit start, task-trained policy paths copy their own version-926 $\rho$. Case memory is initialized from each repository's completed training history. The original and calibrated reference policies first process the pooled training labels through their own fast updates, then fork into independent audit sessions. HEDGE4 also processes the same pooled training history before the audit. Historical supervision is therefore available to every comparator through its declared learning mechanism.

HEDGE4 is the four-expert comparator constructed for this study, using an exponential-weights mixture~\cite{freund1997decision}. Its experts are a smoothed workflow-frequency prior, a last-label predictor for the workflow/job key, a last-label predictor for the commit/workflow/job key, and hashed-feature logistic regression. It updates mixture weights using prediction-time losses after labels arrive. We evaluate this delayed-feedback combination by its observed prediction loss. Every method receives the same externally available history; their computation costs differ.

This replay extends the lifetime of the fast state. In the episodic study, two updates precede one designated query. Here each completed batch can update the same factors that will be used for many later jobs. The distribution of arriving labels can differ from the distribution of jobs currently starting, and their order is set by completion times. A locally fitted update can therefore carry evidence from a narrow or delayed part of the workload into a different prediction population.

\subsection{Prediction quality and retained state}
The primary prediction loss is the Brier score~\cite{brier1950verification} with a factor of $1/2$ normalization (half-Brier),
\begin{equation}
 b(p,y)=\tfrac12\sum_{c=1}^{4}(p_c-\mathbf1[c=y])^2.
\end{equation}
We average jobs within commits, commits within repositories, and the two repositories equally. Seed means and sample standard deviations are reported separately. The primary success rule requires HEDGE4 minus $P_1M_1$ to average at least 0.01 and be positive for all three seeds.

\begin{table}[ht]
\centering\small
\begin{tabular}{lrrrrr}\toprule
Path & 2811 & 2812 & 2813 & Mean & Seed SD\\\midrule
\texttt{P0M0} & 0.1291 & 0.1196 & 0.1372 & 0.1286 & 0.0088 \\
\texttt{P1M0} & 0.1667 & 0.1542 & 0.1970 & 0.1726 & 0.0220 \\
\texttt{P0M1} & 0.1162 & 0.1174 & 0.1487 & 0.1274 & 0.0184 \\
\texttt{P1M1} & 0.1839 & 0.1658 & 0.1927 & 0.1808 & 0.0137 \\
\texttt{PERM\_M1} & 0.3193 & 0.2867 & 0.3854 & 0.3305 & 0.0503 \\
\texttt{ORIGINAL} & 0.1681 & 0.1646 & 0.1633 & 0.1653 & 0.0025 \\
\texttt{CALIBRATED} & 0.1576 & 0.1594 & 0.1534 & 0.1568 & 0.0030 \\
\texttt{HEDGE4} & 0.1024 & 0.1024 & 0.1024 & 0.1024 & 0.0000 \\
\bottomrule\end{tabular}
\caption{Full CI replay, grouped half-Brier loss. \texttt{PERM\_M1} uses permuted fast-update labels and real memory. \texttt{ORIGINAL} retains the pre-CI slow initialization; \texttt{CALIBRATED} adds a four-class bias fitted on historical data. Both controls also receive historical and online fast updates.}
\label{tab:ci}
\end{table}
The full path scores 0.1808 compared with HEDGE4's 0.1024, a gain of $-0.0784$ that is negative for every seed. Reading at the trained initialization with real memory scores 0.1274. Consequently, retaining the real-feedback numerical state has $G=-0.0534$. Real feedback still outperforms permuted feedback by $D=0.1497$. These measurements separate two observations: correct feedback is less harmful than permuted feedback, while both continuous numerical updating and the full path remain worse than their respective comparators.

Without explicit memory, retention also increases loss, from 0.1286 to 0.1726. The memory effect and policy-by-memory interaction vary in sign across seeds. NumPy contributes one all-success commit; the equal-repository average gives that single group half the weight. A pandas commit receives weight $1/142$, so the single NumPy commit carries 71 times its weight. We report this queue descriptively, without treating training seeds as additional independent jobs or attaching a population confidence interval to the two-repository result.

\begin{table}[htbp]
\centering\small
\begin{tabular}{llrrrrrr}\toprule
Repository & Label & Jobs & Commits & Reset & Keep & Permuted & HEDGE4\\\midrule
NumPy & All & 83 & 1 & 0.0300 & 0.0273 & 0.2656 & 0.0277 \\
NumPy & Success & 83 & 1 & 0.0300 & 0.0273 & 0.2656 & 0.0277 \\
pandas & All & 3238 & 71 & 0.2248 & 0.3343 & 0.3954 & 0.1772 \\
pandas & Success & 2310 & 63 & 0.0401 & 0.2048 & 0.3926 & 0.0702 \\
pandas & Failure & 32 & 12 & 0.9386 & 0.7644 & 0.3842 & 0.7128 \\
pandas & Cancelled & 704 & 18 & 0.8475 & 0.6987 & 0.4050 & 0.5847 \\
pandas & Other & 192 & 57 & 0.2692 & 0.7327 & 0.4179 & 0.0574 \\
\bottomrule\end{tabular}
\caption{Saved CI half-Brier losses by repository and observed label, averaged over the three seeds. Rows labeled All weight commits equally within a repository; individual-label rows weight jobs equally within that label. Reset, Keep, and Permuted all use real memory. Label strata are descriptive and overlap in their commit membership.}
\label{tab:ci-strata}
\end{table}
The decomposition locates the retained-state harm primarily in pandas (Table~\ref{tab:ci-strata}). Keep slightly improves the one all-success NumPy commit, from 0.0300 to 0.0273. For pandas it raises commit-weighted loss from 0.2248 to 0.3343. Its job-weighted losses are lower on failures and cancellations but higher on successes and other outcomes. These strata describe where the error occurs; they do not identify which update caused it.

\begin{table}[htbp]
\centering\small
\begin{tabular}{lrrrr}\toprule
Aggregation & Reset & Keep & Permuted & HEDGE4\\\midrule
Repository then commit (primary) & 0.1274 & 0.1808 & 0.3305 & 0.1024 \\
Commit (descriptive) & 0.2221 & 0.3300 & 0.3936 & 0.1751 \\
Job (descriptive) & 0.2329 & 0.3410 & 0.3934 & 0.1836 \\
\bottomrule\end{tabular}
\caption{Weighting sensitivity on the same saved 3,321 predictions per seed. The first row remains the registered primary measure. The other rows are descriptive; neither introduces new jobs, retraining, or a replacement success rule.}
\label{tab:ci-weighting}
\end{table}
The absolute scores depend on repository weighting, but the ordering HEDGE4, reset, keep, permuted is preserved under all three aggregations (Table~\ref{tab:ci-weighting}). Thus the observed harm of continuous retention is not removed by weighting commits or jobs globally.

\section{Separating Update Magnitude from Parameter History}
\label{sec:diagnostic}
The CI replay records when predictions begin to degrade but does not retain every intermediate gradient or support loss. We therefore conducted one fixed follow-up intervention using the same trained slow states and a declared early prefix. It is a mechanism study on the already observed audit queue.

The intervention crosses two choices. \texttt{CARRY} preserves the fast parameters across feedback batches; \texttt{LATEST} restores them to the bound $\rho$ before each batch. The second factor either applies the original SGD increment or multiplies its actual FP32 parameter difference by 0.1. We denote these settings by \texttt{1} and \texttt{TENTH}. The latter is defined on the realized parameter difference, rather than assumed to be bitwise identical to another learning-rate implementation. Memory, arrived feedback, and private random streams retain their prescribed lifetimes. \texttt{RESET} always reads the bound slow initialization.

The prefix includes all 83 NumPy jobs and 653 pandas jobs, totaling 736 predictions per seed. By its cutoff, 422 pandas labels have arrived in 128 batches; 231 remain pending. The pandas prefix intersects 13 commits, of which six are complete. We restrict the original full-queue job weights to this prefix and renormalize within each repository. This produces a prefix measure, not a new full-group test set. Saved HEDGE4 predictions are reused on the same jobs. The repeated \texttt{CARRY\_1} and \texttt{RESET} probabilities exactly match the original saved probabilities for all seeds.

\begin{table}[ht]
\centering\small
\begin{tabular}{lrrrrr}\toprule
Update path & 2811 & 2812 & 2813 & Mean & Seed SD\\\midrule
\texttt{CARRY\_1} & 0.3577 & 0.1925 & 0.3400 & 0.2967 & 0.0907 \\
\texttt{CARRY\_TENTH} & 0.2847 & 0.3384 & 0.3382 & 0.3204 & 0.0310 \\
\texttt{LATEST\_1} & 0.2666 & 0.2488 & 0.3115 & 0.2756 & 0.0323 \\
\texttt{LATEST\_TENTH} & 0.2223 & 0.1748 & 0.2713 & 0.2228 & 0.0482 \\
\texttt{RESET} & 0.1775 & 0.2214 & 0.2302 & 0.2097 & 0.0282 \\
\texttt{HEDGE4} & 0.1679 & 0.1679 & 0.1679 & 0.1679 & 0.0000 \\
\bottomrule\end{tabular}
\caption{Fixed-prefix diagnostic, half-Brier loss. The primary comparison is \texttt{CARRY\_1} minus \texttt{CARRY\_TENTH}. All paths use the same prefix. Their losses use a different evaluation population from Table~\ref{tab:ci}.}
\label{tab:diagnostic}
\end{table}
Reducing the carried update has a mean gain of $-0.0237$, with per-seed gains 0.0730, $-0.1459$, and 0.0017. It fails the fixed rule of a mean gain of at least 0.01 with all seeds positive. At the smaller increment, removing accumulated parameter history improves loss by 0.0976 on average and for every seed. At the original increment, the history-removal effect changes sign across seeds. All four active paths have higher mean loss than \texttt{RESET}, and each is worse than HEDGE4. \texttt{LATEST\_TENTH} has the lowest mean loss among the four active paths, followed by \texttt{LATEST\_1}, \texttt{CARRY\_1}, and \texttt{CARRY\_TENTH}.

\subsection{What happens on the first feedback batch?}
The first pandas feedback batch contains eight jobs labeled other. Three more other labels arrive in the next batch. At 01:16:06 UTC, after those two updates, all three original-update seeds have a local retained-state gain below $-0.01$ on the same 37 newly starting jobs. Those jobs later resolve to 24 successes and 13 cancellations. This is a descriptive local marker; the first cumulative crossing occurs after 2, 421, and 25 feedback batches for the three seeds, respectively.

\begin{table}[ht]
\centering\small
\begin{tabular}{lrrrr}\toprule
Seed & Before update & Original increment & Tenth increment & Original increment norm\\\midrule
2811 & 0.513218 & 3.007320 & 0.102798 & 1.590172 \\
2812 & 2.503274 & 2.838060 & 2.805870 & 7.709818 \\
2813 & 0.721480 & 1.350831 & 0.139986 & 2.632058 \\
\bottomrule
\end{tabular}
\caption{Cross-entropy on the same first pandas feedback batch, before and after the numerical update. Increment norm is the Euclidean norm across the two factor tensors.}
\label{tab:first}
\end{table}
The original update increases same-batch cross-entropy for all three seeds (Table~\ref{tab:first}). At that point, the feedback itself is not being fit better. The observation identifies local non-descent under the implemented update, which is distinct from a support-loss decrease followed by poor generalization. The smaller increment reduces same-batch loss for two seeds and increases it for the third. The pre- and post-update losses use the same saved prompts, candidate labels, masks, positions, and per-path batch-mean normalization. The pre-update gradient traverses the mixer with non-reentrant activation checkpointing; the post-update measurement uses a no-gradient forward. The substrate stays in evaluation mode, with BF16 computation and FP32 factors. These are matched loss definitions, not a guarantee of identical numerical kernel paths. Step size, factor geometry, model curvature, and finite-precision arithmetic have not been independently separated by this intervention.

Across the full diagnostic prefix, the smaller \texttt{CARRY} path has lower average support loss before and after updates, yet worse mean future half-Brier loss. The selected update batch and the later prediction population thus answer different questions. Feedback delay and class composition offer plausible mechanisms for this divergence, but their individual causal contributions require different interventions. A support-loss decrease alone is insufficient evidence that an online update is useful.

\begin{figure}[htbp]
\centering\includegraphics[width=\linewidth]{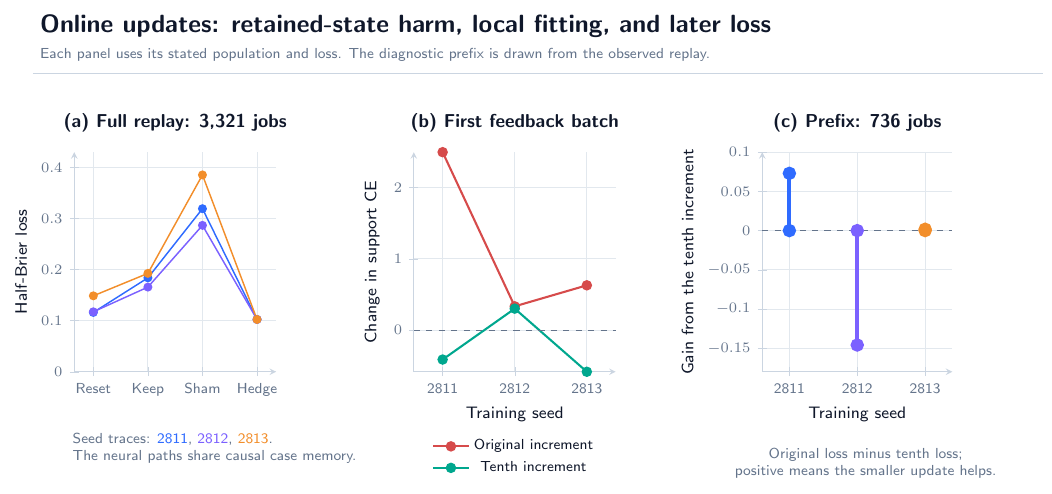}
\caption{\textbf{Three views of online learning.} (a) The full 3,321-job replay compares real-memory paths at reset, retained real-feedback state, permuted-feedback state, and HEDGE4. (b) Same-batch CE change on the first eight pandas feedback jobs: positive values indicate non-descent. (c) The fixed 736-job diagnostic prefix: original carried-update loss minus the smaller carried-update loss, separately for each seed. The latter two panels describe a diagnostic on the observed queue; their populations and losses are distinct from panel (a).}
\label{fig:online-results}
\end{figure}

Read together, the panels in Figure~\ref{fig:online-results} separate three stages of an update's effect. The full replay measures the consequence of retaining a learning trajectory. The first-batch measurement asks whether the update improves the batch that generated its gradient. The prefix intervention changes the realized step magnitude while preserving the comparison population. The original step fails locally at the first batch; shrinking it changes that local behavior without producing a consistent future-risk improvement across seeds. This is the specific failure pattern a subsequent update controller has to address.

\FloatBarrier

\section{Implementation, Cost, and Reproducibility}
The episodic run uses PyTorch 2.11.0, Transformers 5.12.1, CUDA 12.8, and one RTX 4090. The substrate is BF16 and LARC factors are FP32. A physical batch has eight sequences of length 512, including padding. Prompt and answer limits are 384 and 32 tokens, respectively, and programs are scored with their complete end-of-sequence token. The mean-token scoring convention is used in both studies.

\begin{table}[ht]
\centering\small
\begin{tabularx}{\linewidth}{YrrY}\toprule
Study & Slow updates & Entry wall time (s) & Saved learning state\\\midrule
Episodic training and evaluation &1,536&1,634.503&Six slow initializations, version 256\\
CI task fit and complete replay &2,010&13,229.793&Three slow initializations, version 926; six online states\\
CI prefix intervention &0&1,058.808&Six sessions containing all four active fast states\\\bottomrule
\end{tabularx}
\caption{Recorded execution costs. Wall times include model loading, evaluation, CPU work inside the entry, saving, and cleanup. The CI total includes two stopped preflights, an interrupted partial replay, and its completed evaluation. They are not pure GPU kernel times.}
\label{tab:cost}
\end{table}
The episodic preflight also performed two temporary slow updates that were discarded before production. Its total operation ledger contains 11,581 mixer forwards, 6,195 head forwards, and 5,386 backward calls, including preflight, recomputation, and evaluation. The original ledger names its physical and effective position totals \texttt{physical\_tokens} and \texttt{effective\_tokens}: 49,025,792 and 19,521,408. They add positions charged by different operations. Table~\ref{tab:positions} accounts for the physical total exactly. The effective total uses non-padding positions under the same operation accounting. These totals are neither mixer-input tokens alone nor FLOP-equivalent work or independent training examples.

\begin{table}[htbp]
\centering\small
\begin{tabular}{lrrr}\toprule
Operation & Calls & Positions per call & Ledger contribution\\\midrule
Mixer, including recomputation & 11,581 & $8\times512$ & 47,435,776\\
Selected-answer head & 6,195 & $8\times32$ & 1,585,920\\
Native reference & 1 & $8\times512$ & 4,096\\\midrule
Total & & & 49,025,792\\\bottomrule
\end{tabular}
\caption{Decomposition of the episodic physical-position ledger. Head and native-reference charges explain the 1,590,016 positions beyond the mixer subtotal. Backward calls are counted separately.}
\label{tab:positions}
\end{table} Data construction costs and source-review costs are recorded separately from the execution entries.

The CI run encountered a pending-feedback capacity stop after task-fit training. The completed replay reused the saved slow models and restarted the audit from its beginning, counting the repeated evaluation in full. Pending records were compacted to store the two distinct prompt strings, plain and remembered, once per job. The completed configuration allows 2 MiB of pending feedback and 4 MiB of total logical live state per session, including factors, private random streams, memory, and pending metadata. No replacement training was performed to reconstruct their identities. Audit sessions were saved and restored after at least $\min(128,\max(1,\lfloor N_{\rm repo}/2\rfloor))$ audit labels had arrived. Restoration checked the complete pending queue, numerical states, memory, counters, and random states; final sessions were also reloaded. The prefix intervention performed 2,472 fast-state updates and zero slow updates. Its same-batch post-update measurements are included in its recorded cost.

The reproducibility package records the slow substrate, semantic adapter, input-residual placement, trained factors, optimizer state, and random-number state separately. Source manifests and checkpoint identities disambiguate the several historical residual implementations. The manuscript's numerical tables are generated from saved result artifacts. The paper and source archive are accompanied by an evidence map; access to the frozen base model and the original semantic adapter is required to execute the trained residuals.

\subsection{The state boundary in an implementation}
A minimal implementation needs three operations. \emph{Begin} copies the factors from the requested slow version and initializes the episode's private update state. \emph{Update} consumes only feedback that has arrived, computes its declared loss through the frozen reader, and replaces that episode's fast factors. \emph{Read} evaluates the current prompt and candidates using the designated factors. Keeping these operations distinct makes the update lifetime visible in the code and in the resulting checkpoint.

The inner optimizer here is functional SGD, so there are no momentum tensors to carry between steps. Factors, step count, and the private random stream still belong to the episode. The outer AdamW optimizer owns different state attached to $\rho$. Slow versions prevent an episode from silently resetting to an initialization that changed after it began. This is especially useful when many independent episodes share a physical batch.

At an input boundary, cached embeddings are reusable only while their tokenization and frozen substrate match. Once LARC has changed the representation, the downstream activations depend on $\Phi$. A cached mixer output or key--value state from another factor version would evaluate a different function from Equation~\eqref{eq:residual}. The implementation therefore recomputes the mixer with the current factors, preserves attention masks and positions, and uses activation checkpointing to trade computation for memory during backpropagation.

In the delayed-feedback setting, the pending record is also part of the learning state. It retains the prompt used for the original prediction and the information needed to score the eventual label against that prediction. Rebuilding the prompt from current memory at completion would train on a different observation. Saving only the factor matrices would likewise leave restoration unable to determine which feedback belongs to which earlier read.

\FloatBarrier
\section{The Post-Training Route in Orbit}\label{sec:orbit}
Orbit-1B integrates the LARC policy carrier with authoritative memory in a
pretrained model~\cite{mmlaorbitmodelcard}. The released profile uses the
MiniCPM5-1B-SFT substrate with a trained semantic adapter $\omega$, an input
residual of width $d=1536$ and rank $r=4$, and a mixer with trained Orbit
blocks. The input factor pair occupies a single boundary before the mixer.
This placement preserves the division between the pretrained computation and
the private numerical state that changes its input.

\begin{figure}[tbp]
\centering
\includegraphics[width=\linewidth]{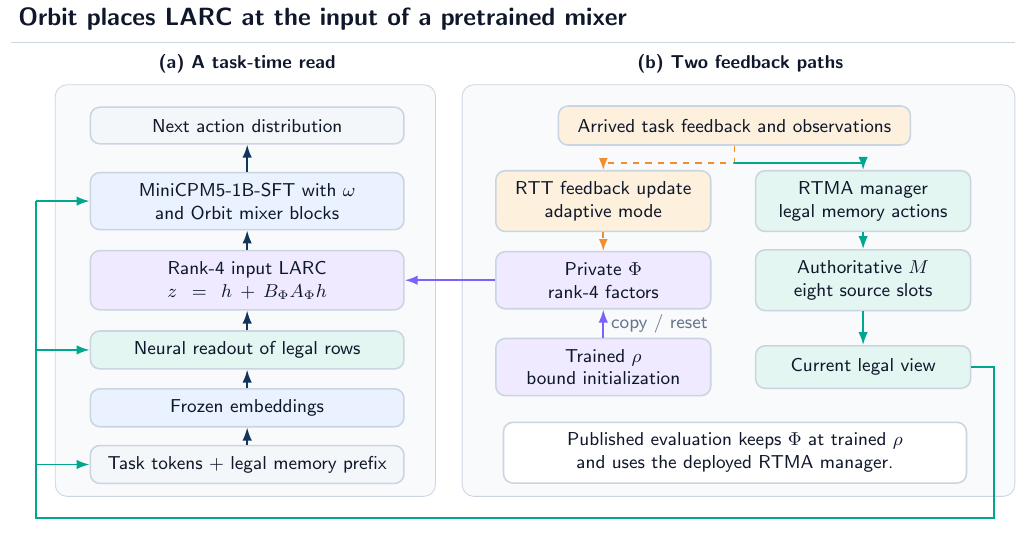}
\caption{The input operator in Orbit. (a) The legal memory prefix joins the
task tokens. The neural input readout acts before the rank-4 residual.
Enabled memory views can also enter the mixer blocks. (b) The dashed path
updates private factors in adaptive mode using the state before the memory
commit. The memory manager processes arrived experience through its own
write boundary. The reported task panels use the trained initialization
throughout each task and allow memory to change.}
\label{fig:orbit-architecture}
\end{figure}

\paragraph{The forward read.}
The reader builds its token input from the current task and the legal
memory prefix. A neural readout can also modify the embedded input before
LARC. The model then applies $T_\Phi$ to those vectors and passes them to the
mixer with the enabled memory views. Figure~\ref{fig:orbit-architecture}
locates the factors at the point where they change the computation of the
next action. The weights of the frozen mixer still transmit gradients to
these factors in a numerical update.

\paragraph{Two state transitions.}
A task starts with a private copy of trained $\rho$. In the adaptive mode,
eligible arrived feedback is bound to the preceding attempt's state. The
numerical update uses that state before the memory commit changes its view.
The memory manager owns the separate transition on $M$. The next read
rebuilds dependent inputs from the resulting factors and legal memory
view. This separates reasoning-time training (RTT), which changes the
numerical policy state, from reasoning-time memory adaptation (RTMA), which
changes the available records. MMLA-RSI specifies how an independently
admitted successor can take over an adaptation rule~\cite{zou2026mmlarsi}.

\paragraph{What the released task panels measure.}
The released evaluation uses \texttt{fixed\_rho} with the deployed RTMA
manager and eight authoritative memory slots~\cite{mmlaorbitmodelcard}.
Orbit succeeds on 56 of 134 ALFWorld tasks on the seen regression panel,
0 of 200 BFCL multi-turn base tasks, and 1 of 90 ScienceWorld tasks. The mean
native ScienceWorld score is $-28.2778$. These outcomes describe the complete configuration
with trained but fixed $\Phi$. They combine the reader, retained records,
deployment manager, and task protocol. The compared external systems use
their own templates and runtime protocols.

To measure the numerical update at a common read boundary, fix the query,
arrived context, legal $M$, and read randomness, and compare keeping $\Phi$
with resetting to its bound $\rho$. Derive each read's inputs under its
own factors. This extends the logic of Section~\ref{sec:causal} to a reader
that also consumes memory. Because the common memory can already reflect
earlier decisions, the contrast describes the next read at that snapshot.

\section{Discussion}
The program-selection study shows that feedback written into this trained rank-4 input residual reduces expected query execution error with the query read otherwise fixed. The target is already among four supplied programs, and the direct support-loss rule remains much more accurate. Within that setting, the retained factors provide a measurable path from feedback to a better probability distribution. Comparing final errors separately from gains over reset also clarifies why both objectives support adaptation, while the fixed worst-seed stability rule selects static training despite its higher mean endpoint error.

The continuous-integration study changes the learning regime from two updates on a small controlled support set to accumulated updates with delayed labels. In this repository-local retrospective queue, the repository-balanced result favors the trained reset reference over retained online state, and HEDGE4 remains stronger than the full neural path. The follow-up intervention records both same-batch non-descent and a lack of consistent future benefit from shrinking updates. It makes update size and parameter history concrete design questions, while leaving the causes of failure unresolved. Rank alone controls neither local descent nor usefulness on later predictions.

The available evidence covers one frozen language-model substrate, one rank, finite candidate decisions, and three training seeds. Comparisons with untrained initialization on the same task groups, alternative residual ranks and insertion boundaries, and matched weight-space LoRA or nonlinear adapters remain unrun. The results characterize the evaluated LARC realization and do not measure a performance advantage over those alternatives. The program results are development evidence; the CI intervention reuses an observed audit population.

A learning-state specification must describe both the current prediction and the next update. Resuming factor-space learning requires the factors and any optimizer state; the product $BA$ describes the current linear correction alone. A trained reset reference measures the effect of retaining updates, and a strong direct-feedback comparator measures how useful the resulting policy is. Within MMLA, these findings motivate evaluating a controller that uses arrived-batch fitting and later outcomes to decide whether to retain an update. Such a controller remains a next study, with independent task groups and the existing comparators needed to assess it.

\section{Relation to Existing Methods}
\paragraph{Residual learning and compact adaptation.}
Residual networks learn a correction relative to a reference path~\cite{he2015deepresiduallearningimage}. ReZero initializes residual branches through a zero-valued scalar gate~\cite{bachlechner2020rezeroneedfastconvergence}. Adapter modules train compact additions to a frozen pretrained network~\cite{houlsby2019parameterefficienttransferlearningnlp}. LoRA parameterizes weight corrections with low-rank factors~\cite{hu2021loralowrankadaptationlarge}; Section~\ref{sec:algebra} gives the linear relationship to an input-side residual and identifies the insertion boundary used here.

\paragraph{Fast weights and learning an initialization.}
Fast weights provide an intermediate timescale between neural activities and slowly trained parameters~\cite{ba2016usingfastweightsattend}. MAML trains initial parameters for rapid task adaptation~\cite{finn2017modelagnosticmetalearningfastadaptation}, and its first-order approximation omits support-loss Hessians~\cite{nichol2018firstordermetalearningalgorithms}. Equation~\eqref{eq:fo} applies that estimator to the input-residual factors, with ordinary query-risk training as the matched static objective.

Block et al.~\cite{block2025provablemetalearninglowrankadaptations} study retraining objectives that prepare model weights for later low-rank adaptation and give performance guarantees in a linear setting. ABMLL~\cite{zhang2026metalearningscalelargelanguage} uses low-rank forms to model global and task-specific parameter distributions for Bayesian adaptation across datasets. Here the pretrained substrate is fixed, the learned initialization is a deterministic pair of input-residual factors, and the experiments measure feedback-retention effects with real/sham and keep/reset interventions. These differences specify the objects trained and evaluated in the respective studies.

\paragraph{Learning during inference.}
Test-time training adapts model parameters through a self-supervised objective at test time~\cite{sun2020testtimetrainingselfsupervisiongeneralization}. TTT layers instead make the sequence state a model that is updated by learning~\cite{sun2025learninglearntesttime}. Titans develops neural memory that learns historical context~\cite{behrouz2024titanslearningmemorizetest}. Our experiments use external execution feedback or completed-job labels, with state resets that isolate their effect. Within MMLA, LARC occupies the numerical policy plane; explicit stored records retain a separate memory lifetime.

\section{Conclusion}
LARC realizes a numerical policy state with a trained starting point,
private feedback updates, and a defined reset. Orbit puts this state before
its pretrained mixer and gives authoritative memory a separate writer.
The controlled program-selection and CI studies show opposite effects of
retaining updates relative to the trained reference. A complete evaluation
of the post-training route therefore needs both the next-read comparison
at a fixed evidence state and the task behavior when the memory manager
continues to act.

\bibliographystyle{plain}
\bibliography{references}

\appendix
\section{Derivation of Factor-Space Dynamics}
For a differentiable scalar loss, $d\ell=\langle G_C,dC\rangle_F$ and $dC=(dB)A+B(dA)$. Cyclic permutation in the trace gives Equation~\eqref{eq:gradient}. A simultaneous gradient step at any $(A,B)$ yields
\begin{align}
 A_+&=A-\eta B^\top G_C,\quad B_+=B-\eta G_C A^\top,\\
 B_+A_+&=BA-\eta BB^\top G_C-\eta G_C A^\top A+\eta^2G_C A^\top B^\top G_C.
\end{align}
Substituting $(cA,B/c)$ proves Equation~\eqref{eq:gauge}. The effective matrix step has two state-dependent preconditioning terms. In particular, two states with identical $BA$ need not have identical subsequent products. Optimizer moments introduce further coordinate dependence.

For the first-step expectation, write $(A_0^\top A_0)_{ij}=\sum_{q=1}^r A_{qi}A_{qj}$. Independence and zero means make off-diagonal expectations zero, while each diagonal expectation is $r\sigma^2$. Because the initial residual is exactly zero, the downstream matrix gradient is independent of the sampled factors whenever the loss depends on them only through $C$. Factor regularizers or factor-dependent stochastic operations would require a separate treatment.

Equation~\eqref{eq:descent} follows from $L$-smoothness:
$\ell(\Phi+u)\le\ell(\Phi)+\langle\nabla\ell(\Phi),u\rangle+\frac L2\|u\|^2$,
with $u=-\eta\nabla\ell(\Phi)$. The observed non-descent in Table~\ref{tab:first} concerns the actual finite step and numerical implementation; it does not contradict this conditional bound.

\section{Episodic Learning Procedure}
\begin{enumerate}
\item Initialize paired slow factors with the same nonzero Gaussian $A$ and zero $B$ for each training seed. Fix the substrate, data split, candidate order, losses, optimizers, and update counts.
\item For each outer batch, bind the current version of $\rho$ and create two private fast states. Obtain the same support feedback and query targets for both training objectives.
\item Evaluate all four candidate programs and perform two support-risk SGD updates per fast state. Preserve episode isolation across sequence rows and random-number streams.
\item Evaluate the static query loss at $\rho$ or the adapted query loss at Equation~\eqref{eq:fo}. Aggregate the two episodes, apply one clipped AdamW update to $\rho$, and increment its version.
\item Save actual factors and optimizer/random state at declared boundaries. Complete all six trajectories before producing development predictions.
\item Evaluate real/sham feedback crossed with keep/reset. Save predictions before opening development query targets. Compute group-weighted risks and paired group intervals, then apply the fixed selection rule.
\end{enumerate}

\section{Task and Protocol Details}
\subsection{Finite parameter support}
Arithmetic scales range over $\{-7,\ldots,-1,1,\ldots,7\}$ and offsets over $\{-12,\ldots,12\}$. Clipping bounds are fixed at $[-16,16]$. List minima range over $\{-6,\ldots,6\}$ and limits over $\{1,\ldots,6\}$. A suffix is assigned by a fixed group-specific random seed, but it does not define a new independent active-parameter group. Arithmetic inputs are integers in $[-48,48]$; list inputs have six entries in $[-12,12]$. Support and query inputs are sampled without duplicate inputs within an episode and with disjoint support/query sets.

Table~\ref{tab:recipes} gives the complete ordered candidate set. Operations are applied from left to right. For arithmetic, \texttt{mul} multiplies by scale, \texttt{add} adds offset, and \texttt{clip} clamps to the fixed lower and upper bounds. For lists, \texttt{filter\_min} preserves entries greater than or equal to minimum in their original order, \texttt{sort} sorts ascending, \texttt{take} retains the first limit entries, and \texttt{append} adds suffix as one final entry. Each parameter group contributes four episodes, one for each possible target recipe; candidate order stays fixed. Support and query outputs are generated by that target program. The target index is withheld from the learner.

\begin{table}[ht]
\centering\small
\begin{tabular}{lll}\toprule
Index & Arithmetic recipe & List recipe\\\midrule
0 & \texttt{mul, add} & \texttt{filter\_min, take}\\
1 & \texttt{add, mul} & \texttt{take, filter\_min}\\
2 & \texttt{clip, mul, add} & \texttt{filter\_min, sort, take, append}\\
3 & \texttt{mul, clip, add} & \texttt{append, filter\_min, sort, take}\\\bottomrule
\end{tabular}
\caption{All eight fixed recipes. A family's four recipes instantiate the four candidate programs using its public parameter values.}
\label{tab:recipes}
\end{table}

The historical exclusion operates on the effective parameter tuple, rather than a newly assigned episode identifier. This prevents relabeling an old transformation as a new task. Model-pretraining overlap is unknown.

\subsection{What reset removes}
In the episodic intervention, reset is applied after the second inner update and before the deterministic query read. It restores the episode's trained $\rho$, clears inner optimizer and cached fast-dependent representations, and restores the private episode random state. Real/reset and sham/reset therefore have equal predictions. The trained $B_\rho$ is retained.

In CI, $P_0M_1$ reads the task-trained initialization with the same causal case memory as $P_1M_1$. The \texttt{LATEST} diagnostic resets only fast parameter history before each feedback batch; it preserves memory and the prescribed random stream. Its post-update state is then read until the next batch. This differs from the always-reset control and from resetting explicit memory.

\subsection{CI observations, labels, and baseline parameters}
The prediction interface contains exactly seven fields: repository name, workflow identifier, workflow path, event type, job name, commit identifier, and recorded start time. The prompt asks for the eventual status and serializes these fields with the selected completed-history cases. Each case exposes its workflow identifier, job name, commit identifier, and arrived result. Current-job completion time and conclusion are absent. Four literal answer strings, \texttt{success}, \texttt{failure}, \texttt{cancelled}, and \texttt{other}, are scored with their end-of-sequence token using the mean-log-probability policy.

The source conclusions \texttt{success}, \texttt{failure}, and \texttt{cancelled} map to labels 0, 1, and 2. Conclusions \texttt{timed\_out}, \texttt{action\_required}, \texttt{neutral}, \texttt{skipped}, and \texttt{stale} map to label 3, \texttt{other}. Noncompleted jobs, unrecognized conclusions, and invalid or missing timestamps are excluded. Commit groups shared across training and audit or exposed in the earlier queue are excluded as groups. The public input is the same for the neural predictors and HEDGE4. The neural path scores label strings, whereas HEDGE4 directly predicts the four class probabilities.

HEDGE4 uses 4,096 keys per bounded statistics map, 1,024 hashed unigram/bigram features, logistic learning rate 0.1, and logistic weight decay $10^{-4}$. Mixture weights are proportional to $\exp(-L_j)$, where $L_j$ is the expert's accumulated prediction-time half-Brier loss observed through arrived feedback. A workflow prior uses Laplace smoothing. The latest workflow/job label receives probability 0.7 and each alternative 0.1; the exact commit/workflow/job label receives 0.97 and each alternative 0.01. Missing keys fall back to the prior. Updates sharing a completion timestamp are processed as a batch.

The calibrated neural control fits four fixed additive label biases by weighted historical cross-entropy with an $\ell_2$ coefficient of $10^{-4}$. It uses CPU float64 L-BFGS, at most 100 iterations and 125 function evaluations, then subtracts the mean bias. Fast updates for that control use the calibrated energies; the bias persists across fast-state replacements. This control tests a fitted class bias together with its prescribed adaptation path.

\Needspace{28\baselineskip}
\subsection{Data and artifact access}
CI observations come from public GitHub Actions run and job metadata for
\texttt{numpy/numpy} and \texttt{pandas-dev/pandas}. Repository--commit groups
are the task units. The companion source and evidence bundle gives the
tables, figures, checkpoint versions, and SHA-256 identities. The public
\texttt{MMLA-LRARC4} package provides the original semantic adapter $\omega$,
the three static $\rho_{256}$ tensors, and component loading code under
AGPL-3.0~\cite{mmlalarccomponents}. Its name retains the earlier LRARC
spelling. The adapted $\rho_{256}$ states, later CI $\rho_{926}$ states,
and full continuation records have separate identities and can be requested
from the authors. Orbit's released model is a separate system package
with its own runtime and model lineage.

\begin{table}[H]
\centering\small
\begin{tabularx}{\linewidth}{lYY}\toprule
Artifact & Exact identity in accompanying manifest & Access at report preparation\\\midrule
MiniCPM5-1B-SFT & Revision \texttt{a60b37f1}\ldots; base file and tensor hashes & \href{https://huggingface.co/openbmb/MiniCPM5-1B-SFT/tree/a60b37f1fc409c54e1e337b0723aaac6f92dfec0}{Public upstream revision}\\
Frozen semantic $\omega$ & Original and exported file hashes; tensor identity & Public \texttt{MMLA-LRARC4} component package\\
Three static $\rho_{256}$ & Seed-specific original/export hashes & Public \texttt{MMLA-LRARC4} component package\\
Three adapted $\rho_{256}$ & Original checkpoint hashes and versions & Project archive; author access\\
Three CI $\rho_{926}$ & Original checkpoint hashes; parent lineage & Project archive; author access\\
Scientific implementation & Program-selection source revision \texttt{7a8e7b77}\ldots; CI source revision \texttt{f8a8ed95}\ldots & Private project repository; author access\\
Paper and evidence summaries & Self-contained LaTeX, generated tables, identities & Accompanying source archive\\\bottomrule
\end{tabularx}
\caption{Artifact availability and identity. The upstream model and the component package retain their respective distribution terms.}
\label{tab:access}
\end{table}

\section{Exact Algebra and State Boundaries}\label{app:checks}
The factor-space derivation in Appendix A uses real matrices of compatible
shapes and a loss that depends on the factors through $C=BA$. The companion
algebra check expands the simultaneous update as a noncommutative polynomial,
checks the specialization $B_0=0$, and checks the scalar witness in
Proposition~\ref{prop:product-state}. It also checks the linear overlap in
Equation~\eqref{eq:lora}. The rank, norm, and smoothness statements use the
paper's mathematical assumptions. Mixed-precision gradients and the full
checkpoint protocol require their own execution checks.

At the abstract read boundary, write the state as
$(\rho_v,\Phi,M,X)$, where $X$ fixes arrived context and read controls.
Reset copies $\operatorname{factors}(\rho_v)=(A_{\rho_v},B_{\rho_v})$
into the second coordinate. A predetermined memory replacement $M^\star$
changes the third. Whenever both orders of these fixed operations are legal,
their common result is
$(\rho_v,\operatorname{factors}(\rho_v),M^\star,X)$.
The reader then derives its inputs under that state. An adaptive decision
can choose different rows before and after reset. Resuming the learning
process also requires its optimizer state and pending feedback bindings.
\end{document}